\documentclass{article}
\usepackage{fullpage}
\usepackage[authoryear,sort]{natbib}

\usepackage{comment}

\usepackage[utf8]{inputenc} 
\usepackage[T1]{fontenc}    
\usepackage{hyperref}       
\usepackage{url}            
\usepackage{booktabs}       
\usepackage{amsfonts}       
\usepackage{nicefrac}       
\usepackage{microtype}      

\usepackage{enumitem}
\usepackage{amsmath,amssymb,mathtools}
\usepackage{bm}
\usepackage{amsthm}
\usepackage{graphicx,subfig,color,wrapfig,epsfig,epstopdf}
\usepackage{enumitem}
\usepackage{booktabs}
\usepackage{algorithm, algorithmic}
\usepackage{color}

\usepackage[sc]{mathpazo} 
\usepackage[T1]{fontenc}  
\usepackage{microtype}    
\usepackage{abstract}     

\allowdisplaybreaks[4]

\newcounter{ass_counter}
\newcounter{thm_counter}
\newcounter{remark_counter}
\newtheorem{theorem}[thm_counter]{Theorem}
\newtheorem{lemma}[thm_counter]{Lemma}

\newtheorem{assumption}[ass_counter]{Assumption}
\newtheorem{remark}[remark_counter]{Remark}

\def\MIF{\text{\bf MIU}}

\usepackage{bbm} 

\DeclareMathOperator*{\argmax}{arg\,max}

\newcommand\relphantom[1]{\mathrel{\phantom{#1}}}

\newcommand{\cN}{\mathcal{N}}
\newcommand{\R}{\mathbb{R}}
\newcommand{\E}{\mathbb{E}}

\title{AutoML from Service Provider's Perspective: Multi-device, Multi-tenant Model Selection with GP-EI}

\usepackage{authblk}
\author[1]{Chen Yu\thanks{\texttt{cyu28@ur.rochester.edu}}}
\author[2]{Bojan Karla$\check{\textrm{s}}$\thanks{\texttt{karlasb@student.ethz.ch}}}
\author[3]{Jie Zhong\thanks{\texttt{jie.zhong@calstatela.edu}}}
\author[2]{Ce Zhang\thanks{\texttt{ce.zhang@inf.ethz.ch}}}
\author[1,4]{Ji Liu\thanks{\texttt{ji.liu.uwisc@gmail.com}}}
\affil[1]{Department of Computer Science, University of Rochester}
\affil[2]{Department of Computer Science, ETH Zurich}
\affil[3]{Department of mathematics, California state university Los Angeles}
\affil[4]{Tencent AI Lab}

\begin{document}

\maketitle






\begin{abstract}

AutoML has become a popular service that is provided by
most leading cloud service providers today. In this paper,
we focus on the AutoML problem from the \emph{service provider's perspective}, motivated 
by the following practical consideration: When an
AutoML service needs to serve {\em multiple users}
with {\em multiple devices} at the same time, how can we allocate 
these devices to users in an efficient way? We focus on GP-EI, one of the most popular algorithms
for automatic model selection and hyperparameter tuning,
used by systems such as Google Vizer.
The technical contribution of this paper is the first
multi-device, multi-tenant algorithm for GP-EI that
is aware of \emph{multiple} computation devices
and multiple users sharing the same set of
computation devices. Theoretically, given $N$ users
and $M$ devices, we obtain a regret bound of $O((\text{\bf {MIU}}(T,K) + M)\frac{N^2}{M})$, where $\text{\bf {MIU}}(T,K)$ 
refers to the maximal incremental uncertainty up to time $T$ for the covariance matrix $K$. 
Empirically, we evaluate our algorithm on two applications
of automatic model selection, and show that our
algorithm significantly 
outperforms the strategy of serving users independently.
Moreover, when multiple computation devices are available, we
achieve near-linear speedup when the number of
users is much larger than the number of devices.

\end{abstract}

\section{Introduction}

\begin{figure}
	\centering
	\includegraphics[width=6cm]{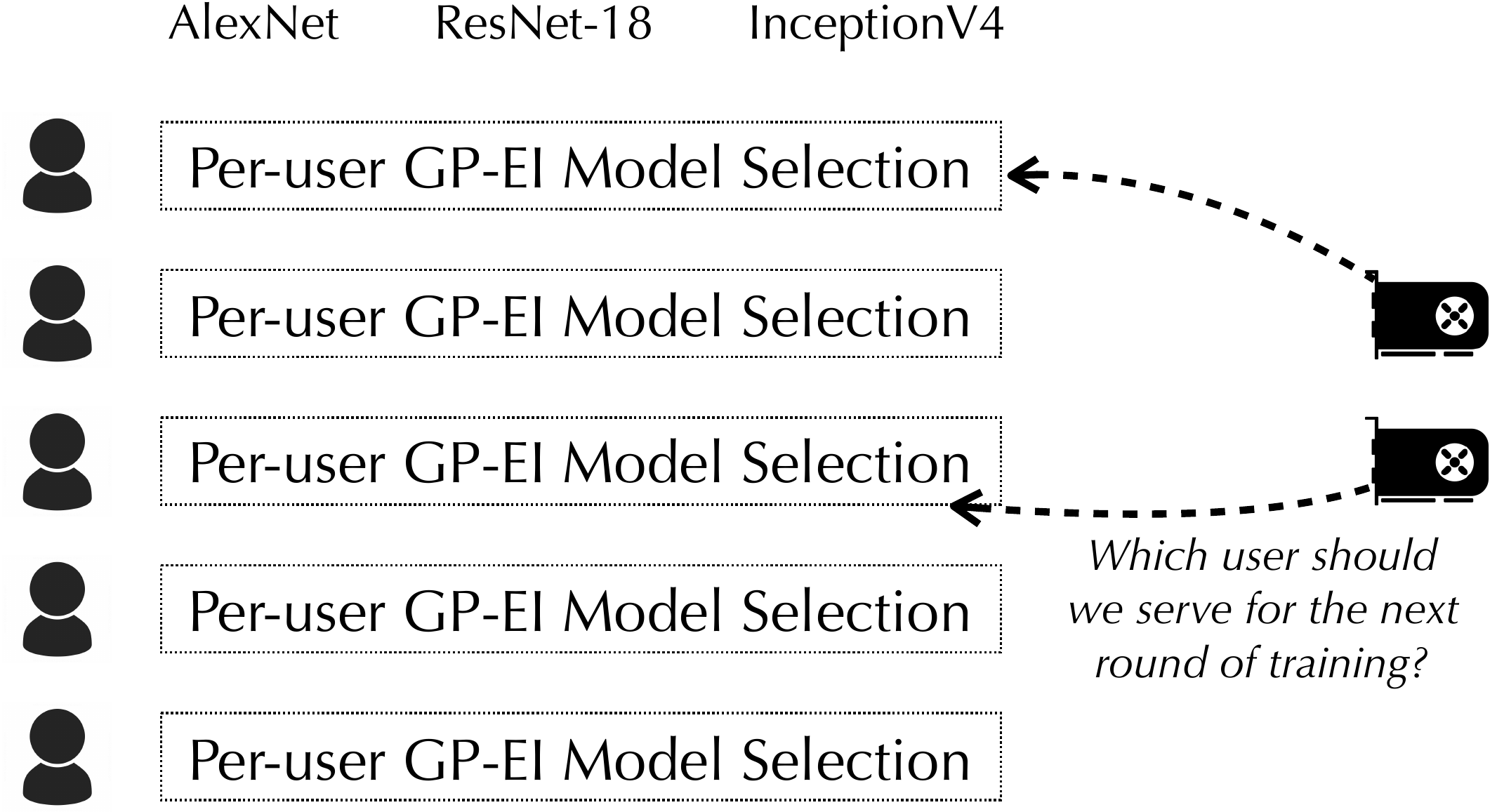}
	\caption{Multi-device, Multi-tenant Model Selection} \label{fig:example}
\end{figure}

One of the next frontiers of machine learning research
is its accessibility --- How can we make machine learning 
systems easy to use such that users do not need to worry
about decisions such as model selection and hyperparameter
tuning as much as today? The industry's answer to this 
question seems to be making AutoML services available on 
the cloud, and prominent examples include Google Cloud 
AutoML and Microsoft Cognitive Services. In these services,
users are provided with a single interface for uploading 
the data, automatically dealing
with hyperparameter tuning and/or model selection, and
returning a model directly without any user intervention as shown in Figure~\ref{fig:example}.
Bayesian optimization is one of the core techniques that
make AutoML services possible by strategically planning
a series of models and hyperparameter configurations 
to tune and try.

From the service provider point of view, resource allocation
is the problem coming together with
the emerging popularity of such a service --- When an online AutoML
service needs to serve multiple users with limited
number of devices, what is the most cost efficient way 
of allocating these devices to different users? During
our conversation with multiple cloud service providers,
many of them believe that an effective and cost efficient
resource sharing could be of great practical interests
and is a natural technical question to ask.


In this paper, 
we focus on GP-EI,
one of the most popular algorithms for AutoML 
that is used in systems such as
Google Vizier~\cite{GoogleVizier} and Spearmint~\cite{Spearmint}. Specifically, we are interested in
the scenarios that each user runs her own GP-EI instance
on a different machine learning task, and there are multiple
devices, each of which can only serve one user at the same time.
{\em How to allocate resources?}
{\em What are the theoretical properties of such an algorithm?}





The result of this paper is the first 
\emph{multi-device}, \emph{multi-tenant}, \emph{cost sensitive} GP-EI algorithm
that aims at optimizing for the ``global happiness'' of all
users given limited resources.
In order to analyze its performance, we introduce a new notation
of Maximum Incremental Uncertainty (\text{\bf {MIU}}) to measure dependence among all models. Given $N$ users and $M$ devices, the upper bound of cumulative regret is $O((\text{\bf {MIU}}(T,K) + M)\frac{N^2}{M})$, where $\text{\bf {MIU}}(T,K)$  will be specified in Section \textbf{5.2}. This bound converges to optimum and is nearly linear speedup when more devices are employed, which will be discussed also in Section \textbf{5.2}.

We evaluate our algorithms on two real-world datasets: (1)
model selection for image classification tasks that contains
$22$ datasets (users) and $8$ different neural network 
architectures; and (2) model selection for models available
on Microsoft Azure Machine Learning Studio that contains 
$17$ datasets and $8$ different machine learning models.
We find that, multi-tenant GP-EI outperforms standard 
GP-EI serving users randomly or in a round robin fashion,
sometimes by up to $5\times$ in terms of the time it
needs to reach the same ``global happiness'' of all users.
When multiple devices are available, it can provide
near linear speedups to the performance of our algorithm.

\section{Related Work}

\paragraph{Multi-armed bandit problem}
There are tons of research on the multi-armed bandit problem. Some early work such as \citet{lowerbound} provided the lower bound of stochastic multi-armed bandit scenario and designed an algorithm to attain this lower bound. Then many research focused on improving constants of the bound and designing distribution-free-bound algorithms such as upper confidence bound (UCB) \citep{ucb} and Minimax Optimal Strategy in the Stochastic case (MOSS) \citep{moss}. UCB is now becoming a very important algorithm in bandit problem. Lots of algorithms are based on UCB, such as LUCB \citep{LUCB}, GP-UCB \citep{GPUCB}. The UCB algorithm constructs the upper confident bound for each arm in every iteration, and chooses the arm with largest bound as the next arm to observe. UCB is very efficient by effectively balancing exploration and exploitation, admitting the regret upper bound $O(\sqrt{T|\mathcal{L}|\log T})$ \citep{budeck}, where $T$ is running time, $\mathcal{L}$ is the set of arms.
There also exist some variations of the bandit problem other than stochastic bandit problem, such as contextual bandit problem \citep{zhangcontextual, contextual2017} and bandit optimization problem \citep{2012optimization,introonline}. We recommend readers to refer to the book by \citep{budeck}. Regret is a common metric of algorithms above, while another important metric is the sample complexity \citep{LUCB, ugape}. Recently there are also some research to combine bandit algorithms and Monte Carlo tree methods to find out an optimal path in a tree with random leaf nodes, that corresponds to finding the optimal strategy in a game. 
Representative algorithms include UCT \citep{UCT}, UGapE-MCTS \citep{bestarm}, and LUCB-micro \citep{structure}.

\paragraph{Expected improvement methods}
The expected improvement method dates back to 1970s, when \citet{EIstart} proposed to use the expected improvement function to decide which arm to choose in each iteration, that is, the arm is chosen with the maximal expected improvement. 
The advantage of this method is that the expected improvement can be computed analytically \citep{EIstart, analytically}. Recently, \citet{Spearmint} extended the idea of expected improvement to the time sensitive case by evaluating the expected improvement per second to make selection. We also adopt this concept in this paper, namely, \textbf{EIrate}. There are also some works on analyzing the asymptotically convergence of EI method. \citet{EI2016} analyzed the hit number of each non-optimal arm and the work by \citet{EI2011} provides the lower and upper bound of instantaneous regret when values of all arms are in Reproducing-Kernel Hilbert Space. Expected improvement methods have many application scenarios, see \citep{EIMS}.



\paragraph{GP-UCB}
GP-UCB is a type of approaches considering the correlation among all arms, while the standard UCB \citep{ucb} does not consider the correlation. GP-UCB chooses the arm with the largest UCB value in each iteration where the UCB value uses the correlation information. The proven regret achieves the rate $O(\sqrt{T\log T\log |\mathcal{L}|\gamma_T})$ where $T$ is running time, $\mathcal{L}$ is the set of arms, and $\gamma_T$ is the maximum information gain at time $T$. Some variants of GP-UCB include Branch and Bound algorithm \citep{GPBB}, EGP algorithm \citep{EGP}, distributed batch GP-UCB \citep{DBGPUCB}, MF-GP-UCB \citep{MFGPUCB}, BOCA \citep{BOCA}, GP-(UCB/EST)-DPP-(MAX/SAMPLE) \citep{EST}, to name a few.

\paragraph{Parallelization bandit algorithms}
To improve the efficiency of bandit algorithms, multiple agents can be employed and they can perform simultaneous investigation. \citet{parallel2017} designed an asynchronous parallel bandit algorithm that allows multiple agents working in parallel without waiting for each other. Both theoretical analysis and empirical studies validate that the nearly linear speedup can be achieved. \citet{parallelthompson} designed an asynchronous version of Thompson sampling algorithm to solve parallelization Bayesian bandit optimization problem. While they consider the single user scenario, our work extends the setup to the multi-user case, which leads to our new notation to reflect the \emph{global happiness} and needs some new technique in theoretical analysis.

\paragraph{AutoML}
Closely related to this work is the emerging trend of AutoML system
and services. Model
selection and hyperparameter tuning is the key technique behind such
services. Prominent systems include
Spark TuPAQ~\citep{Sparks2015}, Auto-WEKA~\citep{Kotthoff2017,Thornton2013}, 
Google Vizier~\citep{Golovin2017}, Spearmint~\citep{Spearmint}, GPyOpt~\citep{gpyopt2016},
and Auto scikit-learn~\citep{Feurer2015} and prominent 
online services include Google Cloud AutoML and Microsoft Cognitive Services.
Most of these systems focus on a single-tenant setting. Recently,
\citet{system} describes a system for
what the authors call ``multi-tenant model selection''.
Our paper is motivated by this multi-tenant setting, however, 
we focus on a much more realistic choice of algorithm, GP-EI, 
that is actually used by real-world AutoML services.

\section{Mathematical Problem Statement}


In this section, we give the mathematical expression of the problem. We first introduce the multi-device, multi-tenant (MDMT) AutoML problem, which is a more general scenario of single-device, single-tenant (SDST) AutoML problem. Then in Section \ref{sec:TSHB}, we propose a new notion -- time sensitive hierarchical bandit (TSHB) problem to abstract the MDMT AutoML problem.  At last in Section \ref{sec:newregret}, we define a new metric to quantify the goal.

\paragraph{Single-device, single-tenant (SDST) AutoML} AutoML often refers to the single-device, single-tenant scenario, in which a single user aims finding the best model (or hyper parameter) for his or her individual dataset as soon as possible. Here the device is an abstract concept, it can refer to a server, or a CPU, GPU. The problem is usually formulated into a Bayesian optimization problem or a bandit problem using Gaussian process to characterize the connection among different models (or hyper parameters).   

\paragraph{Multi-device, multi-tenant (MDMT) AutoML} The MDMT AutoML considers the scenario where there are multiple devices available and multiple tenants who seek the best model for each individual dataset (one tenant corresponds to one dataset). While the objective of the SDST AutoML is purely from the perspective of a customer, the objective of MDMT AutoML is from the \emph{service provider}, because it is generally impossible to optimize performance for each single one, given the limited computing resource. There are two fundamental challenges:
1) When there are multiple devices are available, how to coordinate all computing resources to maximize the efficiency? Simply extending the algorithms in SDST AutoML often let multiple devices to run the same model on the same dataset, which apparently wastes the computing resource;
2) To serve multiple customer, we need to find a global metric to guide us to specify the most appropriate customer to serve besides of choosing the most promising the model for his / her dataset for each time.
The overall problem is how to utilize all devices to achieve a certain global happiness. Each device is considered to be atomic, that is, each device can only run one algorithm (model) on one dataset at the same time.

%
%

\subsection{Time sensitive hierarchical bandit (TSHB) problem} \label{sec:TSHB}

To find a systematic solution to the MDMT AutoML problem, we develop a new notion -- time sensitive hierarchical bandit (TSHB) problem to formulate the MDMT AutoML problem.




\paragraph{Definition of TSHB problem}
Now we formally propose the \textit{ime sensitive hierarchical bandit problem} to abstract the multi-device, multi-tenant autoML framework. Suppose that there are $N$ \emph{users} (or datasets in autoML framework) and $M$ \emph{devices}. Each user has a candidate set of \emph{models} (or algorithms in autoML framework) he or she is interested, specifically, $\mathcal{L}_i$ is the candidate model set for user $i \in [N]$. Here we consider a more general situation, that is we do not assume that $\mathcal{L}_i$ and $\mathcal{L}_j$ are disjoint for $i,j\in [N]$. Denote the set of all models by $\mathcal{L}=\mathcal{L}_1 \cup \mathcal{L}_2 \cup \cdots \cup \mathcal{L}_N$. 
Running a model $x \in \mathcal{L}$ on a device will take $c(x)$ units of time. One model can only be run on one device at the same time and one device can only run one model at the same time. Since one model has been assigned to an idle device, $c(x)$ units time later, the performance of model $x$ will be observed, denoted by $z(x)$. For example, the performance could be the accuracy of the model. W.L.O.G., we assume that the larger the value of $z(x)$, the better. Roughly saying, the overall goal is to utilize $M$ devices to help $N$ users to find out each individual optimal model from the corresponding candidate set as soon as possible. More technically, the goal (from the perspective of service provider) is to maximize the \emph{cumulative global happiness} over all users. We will define a regret to reflect this metric in the next subsection.



\begin{remark}
Here we simply assume $c(x)$ to be known beforehand. Although it is usually unknown beforehand, it is easy to estimate an approximate (but high accurate) value by giving the dataset set size, the computational hardware parameters, historical data, and other information. Therefore, for simplicity in analysis, we just use estimated value so that we can assume the runtime of each model to be exactly known beforehand. In our empirical study, this approximation does not degrade the performance of our algorithm. 
\end{remark}

\vspace{-0.5em}
\subsection{Regret definition for cumulative global happiness} \label{sec:newregret}

\vspace{-1em}
To quantify the goal -- cumulative global happiness, we define the corresponding regret. Let us first introduce some more notations and definitions:
\begin{itemize}[fullwidth]
\item $\mathcal{L}(t)$: the set of models whose performances have been \textbf{observed} up to time $t$;
\item $x^*_i$: the best model for user $i$, that is, $
x^*_i := \argmax_{x\in \mathcal{L}_i} z(x)$; 
\item $x^*_i(t)$: the best model for user $i$ observed up to time $t$, that is,
\begin{equation}\label{currentoptimalx}
x^*_i(t) := \argmax_{x\in \mathcal{L}(t)\cap \mathcal{L}_i} z(x).
\end{equation}
\end{itemize}


We define the individual regret (or negative individual happiness) of user $i$ at time $t$ by the gap between the currently best and the optimal, i.e., $z(x^*_i) - z(x^*_i(t))$.

In most AutoML systems, the user experience goes beyond
the regret at a single time point for a single user.
Instead, the regret is defined by the integration over time and the sum over all users' regrets. It is worth noting that the regret for each user is not the same as the one in the SDST scenario since even a user is not served currently, he or she still receives the penalty (measured by the gap between the optimal model's performance and the currently best performance). More formally, the regret at time $T$ is defined by
\begin{equation} \label{regret}
	\textbf{Regret}_T = \sum\limits_{i=1}^N \int_0^T  \Big(z(x_i^*)-z\big(x_i^*(t)\big)\Big)\text{d}t.
\end{equation}
Our goal is to utilize all devices to minimize this regret.

{\bf Discussion: Why is Multi-Device Important?} 
Having multiple devices in the pool of computation resources
does not necessarily mean that we need to have a multi-device GP-EI algorithm
to do the scheduling. One naive solution, which is adopted by
ease.ml~\cite{system} is to treat all devices as a single
device to do {\em distributed} training for each training task. 
In fact, If the training process can scale up linearly 
to all devices in the computation pool, such a baseline
strategy might not be too bad. However, the availability of
resources provided on a modern cloud is far larger than
the current limitation of distributed training --- Whenever
the scalability becomes sublinear, some resources could be
used to serve other users instead. Given the growth rate
of online machine learning services, we believe the multi-device
setting will only become more important.


\vspace{-0.5em}
\section{Algorithm}

\vspace{-1em}
The proposed Multi-device Multi-tenant GP-EI (MM-GP-EI) algorithm follows a simple philosophy -- as long as there is a device available, select a model to run on this device. To minimize the regret (or equivalently maximize the cumulative global happiness), the key is to select a promising model to run whenever there is an available device. We use the expected improvement rate ({\bf EIrate}) to measure the quality of each model in the set of models that have not yet been selected before (selected models include the ones whose performance have been observed or that are under test currently). The {\bf EIrate} value for model $x$ depends on two factors: the running time of model $x$ and its expected improvement ({\bf EI}) value (defined as Expected Improvement Function in Section \ref{sec:newEI})
\[
{\bf EIrate}(x) := {\bf EI}(x)/c(x).
\]
This measures the averaged expected improvement. This concept also appeared in \citet{Spearmint}.

\subsection{Expected Improvement Function} \label{sec:newEI}

\vspace{-1em}
Every Bayesian-based optimization algorithm has a function called acquisition function \citep{tutorial} that guides the search for the next model to test. In our algorithm, we will call it expected improvement function (EI function).

Suppose at time $t$ these is a device free, we first compute posterior distributions for all models given all current observation and then use these posterior distributions to construct EI function for every model.


First, for each model $x$ and any user who has this model (notice that different users can share the same model), we use $\textbf{EI}_{i,t}(x)$ to denote expected improvement of user $i$'s best performance if model $x$ is observed. Formally, we have 
\begin{equation} \label{EI}
	\textbf{EI}_{i,t}(x) = \E\Big[\max\big\{z(x)-z\big(x_i^*(t)\big),0\big\} \Big].
\end{equation}
Here $\E$ means taking expectation of posterior distribution of $z(x)$ at time $t$.

Then, we sum this value over all users who have model $x$ to represent the total expected improvement $\textbf{EI}_t(x)$ if model $x$ is observed. Formally, we have
\begin{equation}\label{EIsum}
	\textbf{EI}_t(x) = \sum\limits_{i=1}^N \mathbbm{1}(x\in \mathcal{L}_i)\textbf{EI}_{i,t}(x),
\end{equation}
where $\mathbbm{1}(A)=1$ if $A$ happens, and $\mathbbm{1}(A)=0$ if $A$ does not happen. At last, we define the \textbf{EIrate} value of $x$ at time $t$ as follows:
\begin{equation} \label{EIrate}
	\textbf{EIrate}_t(x) = \frac{\textbf{EI}_t(x)}{c(x)}.
\end{equation}

Now, we can choose the model with the max value of EIrate as the next one to run at time $t$:
\begin{equation} \label{next} 
	x_{\text{next to run at time $t$}}= \argmax \limits_{x\in \mathcal{L}\backslash \mathcal{L}(t)} \textbf{EIrate}_t(x).
\end{equation}

\begin{algorithm} [H]
	\caption{MM-GP-EI Algorithm}
	\label{alg1}
	\begin{algorithmic}[1]
		\REQUIRE $\mu(x)$, $k(x,x')$, $c(x)$, $\mathcal{L}$, $\{\mathcal{L}_i \}^N_{i=1}$ and the total time budget $T$.
		\STATE $x^{(i)}_{\text{initial}} = \argmax_{x\in \mathcal{L}_i} \mu(x), \forall i \in [N]$.
		\STATE $\mathcal{L}_{\text{ob}} = \big\{x^{(i)}_{\text{initial}}\big\}_{i=1}^N$
%
		\WHILE {there is a device available and the elapsed time $t$ is less than $T$}
			\STATE Refresh $\mathcal{L}_{\text{ob}}$ to include all observed models at present
			\STATE Update posterior mean $\mu_t(\cdot)$, posterior covariance $k_t(\cdot,\cdot')$ of $z(x)$ given $\{z(x)\}_{x\in \mathcal{L}_{\text{ob}}}$
			\STATE Update $x^*_i(t) = \argmax_{x\in \mathcal{L}_{\text{ob}}\cap \mathcal{L}_i} z(x), \forall i\in [N]$
			\STATE $\textbf{EI}(x) = \sum\limits_{i=1}^N \sum\limits_{x\in \mathcal{L}_i \backslash \mathcal{L}_{\text{ob}}} \sigma_t(x)\tau\Big(\frac{\mu_t(x)-z\big(x_i^*(t)\big)}{\sigma_t(x)}\Big),\forall x\in \mathcal{L}$
			\STATE Run $x_{\text{next}} = \argmax\limits_{x\in \mathcal{L}\backslash \mathcal{L}_{\text{ob}}} \frac{{\textbf{EI}}(x)}{c(x)}$ on this free device
		\ENDWHILE	
		\ENSURE $x^*_1(T), x^*_2(T), \cdots, x^*_N(T)$. 
	\end{algorithmic}
\end{algorithm}

\vspace{-0.5em}
\subsection{Choosing Prior: Gaussian Process}

\vspace{-1em}
Next, we must choose a suitable prior of $z(x)$ to estimate \textbf{EI} function in (\ref{EI}). Here, we choose Gaussian Process (GP) as the prior like many other Bayesian optimization algorithms \citep{GPUCB, EI2011}, mainly because of its convenience of computing posterior distribution and \textbf{EI} function.

A Gaussian Process $GP(\mu(x), k(x,x'))$ is determined by its mean function $\mu(x)$ and covariance function $k(x,x')$. If $z(x)$ has a GP prior $GP(\mu(x), k(x,x'))$, then after observing models in $\mathcal{L}(t)$ at time $t$, the posterior distribution of $z(x)$ given $\{z(x)\}_{x\in \mathcal{L}(t)}$ is also a Gaussian Process $GP(\mu_t(x), k_t(x,x'))$. Here, posterior mean $\mu_t(x)$ and variance $k_t(x,x')$ can be computed analytically. We give the formulas in Supplemental Meterials (\textbf{Section A}).

\textbf{EI} function can also be computed analytically if $z(x)$ obeys Gaussian Process (whatever prior or posterior). The following lemma gives the expression.
\begin{lemma} \label{lemma1}
	Let $\Phi(x)$ denote cumulative distribution function (CDF) of standard normal distribution and $\phi(x)$ denote probability density function (PDF) of standard normal distribution. Also, let $\tau(x) = x\Phi(x) + \phi(x)$. Then, if $X \sim \cN(\mu, \sigma^2)$, and $a\in \R$ is a constant, we have
\begin{equation*}
	\E\Big[\max\big\{X-a,0\big\}\Big] = \sigma\tau\Big(\frac{\mu-a}{\sigma}\Big).
\end{equation*}

\end{lemma}


This section ends by the detailed description of the proposed MM-GP-EI algorithm in Algorithm \ref{alg1}.

\paragraph{Discussion: How to Choose Prior Mean $\mu(x)$ and Prior Covariance $k(x,x')$}
Prior mean $\mu(x)$ and prior covariance $k(x,x')$ are chosen according to the specific problem. They often characterize some properties of models in the problem, such as expected value of models and correlations among different models. In our multi-device multi-tenant example, the parameters of Gaussian process can be obtained from historical experiences and the correlation depends on two factors: the similarity of algorithms and the similarity of users' datasets. We following 
standard AutoML practice (used in Google Vizier or ease.ml)
to construct the kernel matrix from historical runs.

\vspace{-0.5em}
\section {Main Result}

\vspace{-1em}
Before we introduce the main theoretical result, let us propose a new notation \emph{Maximum Incremental Uncertainty} (\MIF), which plays a key role in our theory.

\subsection{Maximum Incremental Uncertainty}
\vspace{-1em}
Suppose that $K$ is the kernel matrix, that is, $K :=[k(x,x')]_{(x,x'\in \mathcal{L})}$, where $k(x,x')$ is the kernel function and $\mathcal{L}$ is the set of all models as defined in Section \textbf{3.1}. So, $K$ is an $|\mathcal{L}|\times |\mathcal{L}|$ positive semi-definite (covariance) matrix. Suppose $S$ is a subset of
$[|\mathcal{L}|] := \{1,2,\cdots, |\mathcal{L}|\}$. Let $K_S$ be a submatrix of $K$ with columns and rows indexed by $S$.

We define the $s$-\MIF\ score of matrix $K$ $(1 \leq s \leq |\mathcal{L}|)$ by


\begin{equation*}
	\MIF_s(K) := \max\limits_{\substack{S'\subset S\subseteq [\mathcal{L}],\\ |S| = s, |S'| = s-1}} 
	\begin{cases}
    \sqrt{\frac{\det(K_S)}{\det(K_{S'})}},& \text{if } \det(K_{S'}) \neq 0;\\
    0,              & \text{otherwise},
\end{cases}
\end{equation*}
where we define $\det(K_{\varnothing}) = 1$. 

Let us understand the meaning of the notation $\MIF$. Given $|\mathcal{L}|$ Gaussian random variables with covariance matrix $K\in \mathbb{R}^{|\mathcal{L}|\times |\mathcal{L}|}$, $\det(K_{S'})$ denotes the total quantity of uncertainty for all random variables in $S'\subset \mathcal{L}$. $\det(K_S) / \det(K_{S'})$ denotes the incremental quantity of uncertainty by adding one more random variable into $S'$ to form $S$. If the added random variable can be linearly represented by random variables in $S'$, the incremental uncertainty is zero. If the added random variable is independent to all variables in $S$, the incremental uncertainty is the variance of the added variable. Therefore, $\MIF_s(K)$ measures the \emph{largest} incremental quantity of uncertainty from $s-1$ random variables to $s$ random variables in $\mathcal{L}$.

\begin{remark}
{\bf Why do not use Information Gain?}
People who are familiar with the concept of information gain (IG) may ask ``IG is a commonly used metric to measure how much uncertainty reduced after observed a sample. Why not use it here?'' 
Although IG and MIU essentially follow the same spirit, the IG metric is not suitable in our setup. Based on the mathematical definition of IG (see Lemma 5.3 in \citet{GPUCB}), it requires the observation noise of the sample to be \emph{nonzero} to make it valid (otherwise it is infinity), which makes it inappropriate in the main target scenario in this paper. In our motivating example -- cloud computing platform, the observation noise is usually considered to be zero, since no people run the same experiment twice. That motivates us to define a slightly different metric, (i.e., Maximum Incremental Uncertainty), to fix the non-observation-noise issue. 
\end{remark}

\vspace{-0.5em}
\subsection{Main Theorem}

\vspace{-1em}
To simplify the analysis and result, we make the following assumption commonly used for analyzing EI \citep{EI2011} and GP-UCB \citep{GPUCB}.

\begin{assumption} \label{assumption}
Assume that
\begin{itemize}[fullwidth]
\item	there exists a constant $R$ such that: $|z(x)-\mu_t(x)|\le \sigma_t(x)R$, for any model $x\in \mathcal{L}$ and any $t\ge 0$; 
\item $\sigma(x) \le 1$.
\end{itemize}
\end{assumption}

Now we are ready to provide the upper bound for the regret defined in \eqref{regret}.

\begin{theorem} \label{main}
Let $\text{\bf MIU}(T,K):= \sum_{s=2}^{|\mathcal{L}(t)|} \MIF_s(K)$. Under Assumption~\ref{assumption}, the regret of the output of Algorithm~\ref{alg1} up to time $T$ admits the following upper bound

	\begin{align*}
	       \text{\bf Regret}_T\lesssim \left(\text{\bf MIU}(T,K) + M\right) \frac{N^2}{M} \bar{c}.
	\end{align*}
	where $\bar{c}:= \frac{1
    }{N}\sum\limits_{i=1}^N c(x_i^*)$ is the average time cost of all optimal models, and $\lesssim$ means ``less than equal to'' up to a constant multiplier.
\end{theorem}

The proof of Theorem \ref{main} can be found in the Supplemental Materials. To the best of our knowledge, this is the first bound for time sensitive regret. 
We offer the following general observations:

\begin{itemize}[fullwidth]
\item ({\bf convergence to optimum}) If the growth of $\text{\bf MIU}(T,K)$ with respect to $T$ is $o(T)$, then average regret converges to zero, that is, 
\[
{1\over T}\text{\bf Regret}_T \rightarrow 0.
\]
In other words, the service provider will find the optimal model for each individual user. 
\item ({\bf nearly linear speedup}) When more and more devices are employed, that is, increasing $M$, then the regret will decrease roughly by a factor $M$ as long as $M$ is dominated by $\text{\bf MIU}(T,K)$.
\end{itemize}

\vspace{-0.5em}
\paragraph{Convergence Rate of the Average Regret.}
We consider the scenario where $M\ll \text{\bf MIU}(T,K)$ and $|\mathcal{L}(t)|$ increases linearly with respect to $T$. Then the growth of $\text{\bf MIU}(T,K)$ will dominate the convergence rate of the average regret. Note that $\text{\bf MIU}(T,K)$ is bounded by
\[
\text{\bf MIU}(T,K) \leq \sum_{i \in \text{top $|\mathcal{L}(t)|$ elements in $\text{diag}(K)$}}\sqrt{K(i,i)}.
\]

\vspace{-0.5em}
\paragraph{Discussion: Special Cases.} 
Consider the following special cases:
\begin{itemize}[fullwidth]
\item ({\bf $O(1/T)$ rate}) The convergence rate for ${1\over T}\text{\bf Regret}_T$ achieves $O(1/T)$, if $\text{\bf MIU}(T,K)$ is bounded, for example, all random variables (models) are linearly combination of a finite number of hidden Gaussian random variables.
\item ({\bf not converge}) 
If all models are independent, then $K$ is a diagonal matrix, and \MIF$_s(K)$ is a constant, which means $\text{\bf MIU}(T,K)$ is linearly increased of $T$. In such a case, the regret is of order $T$, which implies no convergence for the average regret. This is plausible in that the algorithm gains no information from previous observations to decide next because of independence. 
\item ({\bf $O(1/T^{(1-\alpha)})$ rate with $\alpha\in (0,1)$}) This rate can be achieved if $\text{\bf MIU}(T,K) $ grows with the rate $O(T^{\alpha})$.
\end{itemize}

\begin{figure*}[t!]
\centering
\includegraphics[width=\textwidth]{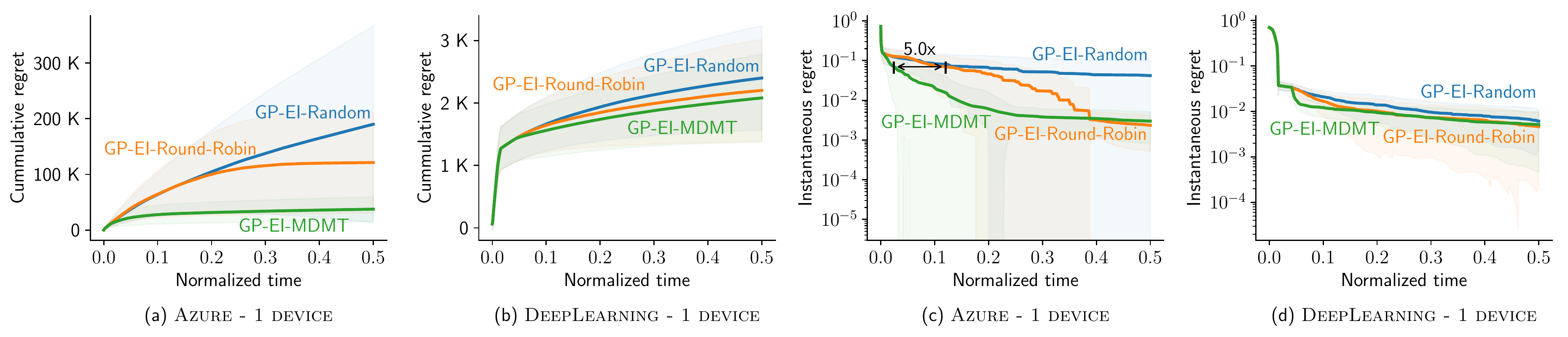}
\vspace{-0.75em}
\caption{Performance of Different Model Selection Algorithms with a Single Computation Device.}
\label{fig:compare-strategies-1-device}
\end{figure*}

\begin{figure*}[t!]
\centering
\includegraphics[width=\textwidth]{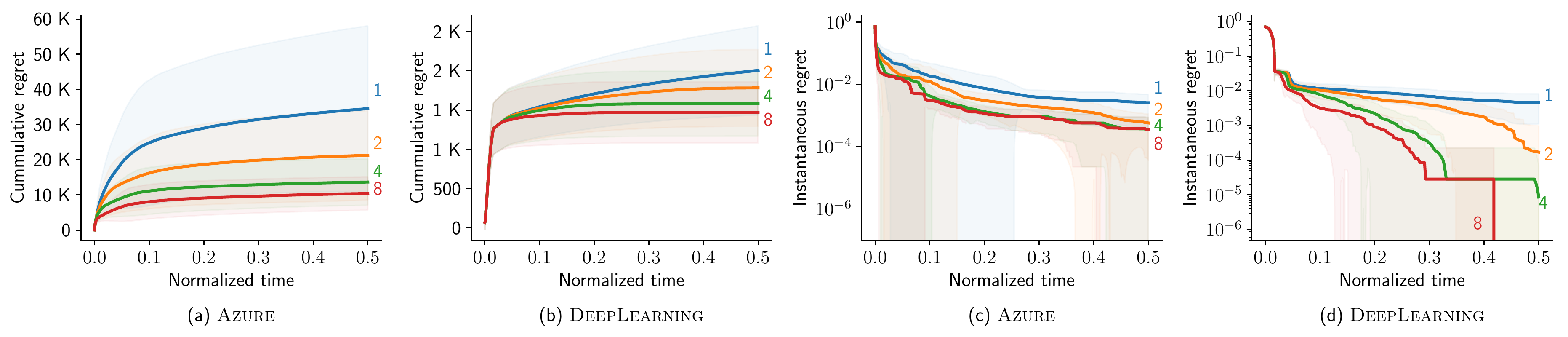}
\vspace{-0.75em}
\caption{The Impact of Multiple Devices on Our Approach.}
\label{fig:compare-devices}
\end{figure*}

\begin{figure*}[t!]
\centering
\includegraphics[width=\textwidth]{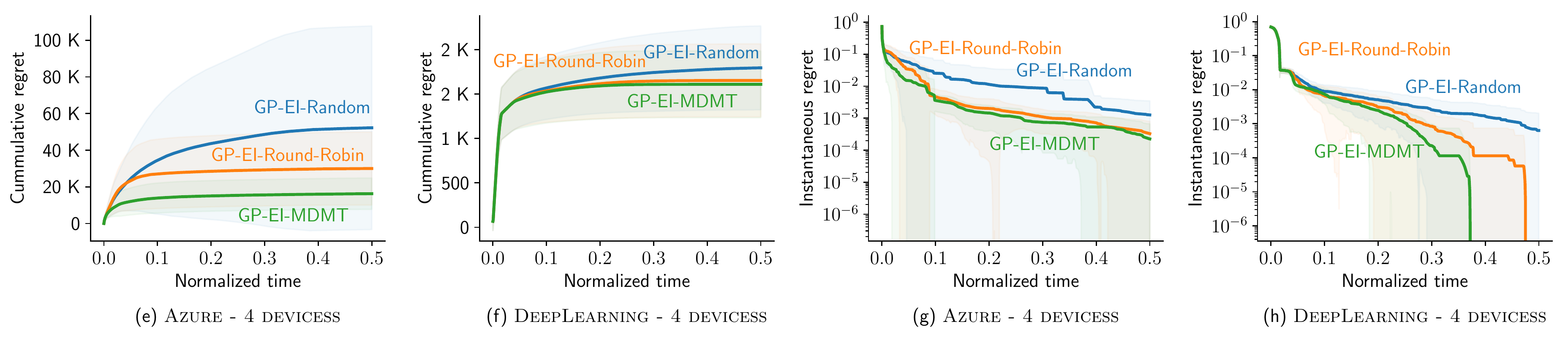}
\vspace{-0.75em}
\caption{Performance of Different Model Selection Algorithms with Four Computation Devices.}
\label{fig:compare-strategies-multi-device}
\end{figure*}

\begin{figure}
	\centering
	\includegraphics[width=0.5\textwidth]{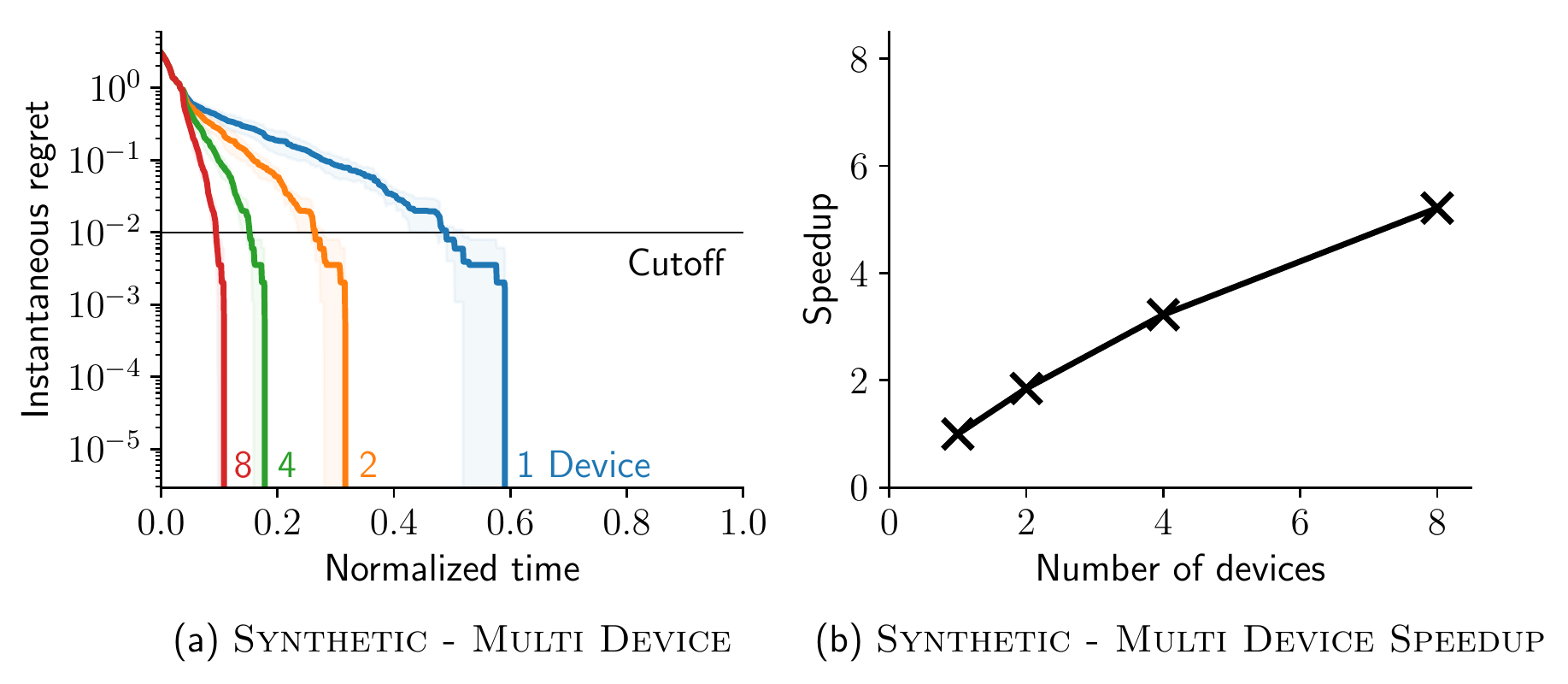}
	\caption{Speedup of using Multiple Devices for Our Approach on Synthetic Data} \label{fig:synthetic-speedup}
\end{figure}

\vspace{-0.5em}
\section {Experiments}

\vspace{-1em}
We validate the effectiveness of the multi-device, multi-tenant 
GP-EI algorithm.

\subsection{Data Sets and Protocol}

\vspace{-0.5em}
We use two datasets for our experiments, namely (1) \textsc{DeepLearning} 
and (2) \textsc{Azure}. Both datasets are from the ease.ml paper~\cite{system}
in which the authors evaluate their single-device, multi-tenant GP-UCB
algorithm. 
The \textsc{DeepLearning} dataset is collected from 22 users, each runs
an image classification task. The system needs to select from 8 deep learning
models, including NIN, GoogLeNet, ResNet-50, AlexNet, BNAlexNet,
ResNet-18, VGG-16, and SqueezeNet.
The \textsc{Azure} dataset is collected from 17 users, each runs
a Kaggle competition. The system needs to select from 8 binary classifiers,
including Averaged Perceptron, Bayes Point Machine, Boosted Decision Tree, Decision Forests,
Decision Jungle, Logistic Regression, Neural Network, and SVM.

\vspace{-0.5em}
\paragraph*{Protocol}
We run all experiments with the following protocol. 
In each run we randomly select 8 users which we will 
isolate and use to estimate the mean and the covariance 
matrix of the prior for the Gaussian process. We test
different model selection algorithms using the 
remaining users. 

For all model selection strategies, we warm start by
training two fastest models for each user, and then
switching to each specific automatic model selection
algorithm:

\begin{itemize}
  \item {\bf GP-EI-Random:} Each user runs their own GP-EI
  model selection algorithm; the system chooses
  the next user to serve uniformly at random and trains one
  model for the selected user;
  \item {\bf GP-EI-Round-Robin:} Each user runs their own
  GP-EI model selection algorithm; the system picks
  the next user to serve in a round robin manner;
  \item {\bf GP-EI-MDMT:} Our proposed approach in which
  each user runs their own
  GP-EI model selection algorithm; the system picks 
  the next user to serve using the MM-GP-EI algorithm we proposed.
\end{itemize}


\vspace{-0.5em}
\paragraph{Metrics}

We measure the performance of a model selection system in
two ways: (1) Cumulative Regret $\textbf{Regret}_T$
and (2) Instantaneous Regret: at time $T$, we calculate the
  average among all users of the difference between the best possible
  accuracy for each user and the current best accuracy the user
  gets. Intuitively, this measures the global ``unhappiness''
  among all users at time $T$.



\vspace{-0.5em}
\subsection{Single device experiments}
\vspace{-0.5em}
We validate the hypothesis that, given a single device, multi-tenant GP-EI 
outperforms both round robin and random strategies for picking the
next user to serve. 
 
\autoref{fig:compare-strategies-1-device} shows the result on
both datasets. The colored region around the curve shows the 1$\sigma$ confidence interval. On \textsc{Azure}, our approach outperforms 
both round robin and random significantly --- we reach the
same instantaneous regret up to 5$\times$ faster than round robin.
This is because, by prioritizing different users with respect to their
expected improvement, the global happiness of all users can 
increase faster than treating all users equally. On the other hand,
for \textsc{DeepLearning}, we do not observe a significant speedup
for our approach. This is because the first two trials of models
already give a reasonable quality. If we measure the 
standard deviation of the accuracy of models for each user, the average for \textsc{Azure} is $0.12$ while for \textsc{DeepLearning} it is $0.04$.
This means that, once the system
trains the first two initial models for each user in \textsc{Azure}, 
there could be more potential performance gain still undiscovered among models 
that have not yet been sampled.


\vspace{-0.5em}
\subsection{Multiple device experiments}
\vspace{-0.5em}
We validate the hypothesis that using multiple devices speeds up 
our multi-device, multi-tenant model selection algorithm. 
\autoref{fig:compare-devices} shows the result of using multiple
devices for our algorithm. We see that
the more devices we use, the faster the instantaneous regret drops.
In terms of speedup, since \textsc{DeepLearning} has
more users than \textsc{Azure} (14 vs. 9), we see that 
the speedup is larger on \textsc{DeepLearning}. The significant speedup of reaching instantaneous regret of $0.03$ for \textsc{Azure} is most probably due to the small number of users compared to the number of devices (9 vs. 8).

We now compare our approach against GP-EI-Round-Robin and 
GP-EI-Random when there
are multiple devices available. \autoref{fig:compare-strategies-multi-device} 
shows the result. We see that, up to 4 devices (9 users in total),
our approach outperforms round robin significantly on \textsc{Azure}. 
When we use 8 devices for Azure,
because there are only 9 users, both our approach and round robin achieve
almost the same performance.

We also conduct an experiment using 
a synthetic dataset with $50$ users and $50$ models (\autoref{fig:synthetic-speedup}). We model the performance as a Gaussian Process and generate random samples independently for each user. The Gaussian Process has zero mean and a covariance matrix derived from the Mat\'ern kernel with $\nu=5/2$. Each generated sample is upwards in order to be non-negative. We run our approach on the same dataset while varying the number of devices. For each device count we repeat the experiment $5$ times. To quantify speed gains we measure the average time it takes the instantaneous regret to hit a cutoff point of $0.01$. We can observe that adding more devices makes the convergence time drop at a near-linear rate.

\vspace{-0.5em}
\section{Conclusion}

\vspace{-1em}
In this paper, we introduced a novel multi-device, multi-tenant algorithm using GP-EI to maximize the ``global happiness'' for all users, who share the same set of computing resources.
We formulated the ``global happiness'' in terms of a cumulative regret and first time provided a theoretical upper bound for the time sensitive regret in the GP-EI framework. We evaluated our algorithm on two real-world datasets, which significantly outperforms the standard GP-EI serving users randomly or in a round robin fashion. Both our theoretical results and experiments show that our algorithm can provide near linear speedups when multiple devices are available.


\bibliography{main}
\bibliographystyle{abbrvnat}

\newpage
\onecolumn
\begin{center}
{\bf \Large Supplemental Materials}
\end{center}

\appendix
\section{Gaussian Process and Posterior Formulas}
In this section we give the posterior formulas of Gaussian Process. Most formulas in this section come from \citet{GPUCB}, but we modify them to fit our settings.

In \textbf{Section 4.2}, we choose the Gaussian Process $GP(\mu(x), k(x,x'))$ as the prior of $z(x)$, and point out that at any time $t$, the posterior of $z(x)$ given $\{z(x)\}_{x\in \mathcal{L}_t}$ is also a Gaussian Process $GP(\mu_t(x), k_t(x,x'))$. Now we give the formulas to compute $\mu_t(x)$ and $k_t(x,x')$.

Suppose at time $t$, the observed models are $x^{(1)}, x^{(2)}, \cdots, x^{(|\mathcal{L}(t)|)}$, then we have the following:
\begin{equation*}
	\mu_t(x) = \boldsymbol{v}_t(x)^\top\boldsymbol{K}_t^{-1}(\boldsymbol{z}_t-\boldsymbol{w}_t)+\mu(x),
\end{equation*}
\begin{equation} \label{v}
	k_t(x,x') = k(x,x')-\boldsymbol{v}_t(x)^{\textrm{T}}\boldsymbol{K}_t^{-1}\boldsymbol{v}_t(x'),
\end{equation}	
where 
\begin{align*}
\boldsymbol{v}_t(x) := & \big[k(x^{(1)},x), k(x^{(2)},x),\cdots, k(x^{(|\mathcal{L}(t)|)},x)\big]^\top, 
\\
\boldsymbol{K}_t := & \big[k(x^{(i)},x^{(j)})\big]_{i,j\in \{1,2,\cdots, |\mathcal{L}(t)|\}},
\\
\boldsymbol{z}_t := & \big[z^{(1)}, z^{(2)}, \cdots, z^{(|\mathcal{L}(t)|)}\big],
\\
\boldsymbol{w}_t := & \big[\mu(x^{(1)}), \mu(x^{(2)}), \cdots, \mu(x^{(|\mathcal{L}(t)|)})\big].
\end{align*}

In this supplemental material, we give the proof of our main theorem. Before that, let us show Lemma \ref{lemma1}.

\section{Proof of Lemma \ref{lemma1}}
\begin{proof}
	First, because $\tau'(x) = \Phi(x)$ and $\tau(-\infty) = 0$, we have $\tau(x) = \int_{-\infty}^x \Phi(t)\text{d}t$. 
	Define $Y = \frac{X-\mu}{\sigma}$, then $X=\sigma Y + \mu$ and $Y\sim \cN(0,1)$.
	Let $S(x) = 1-\Phi(x)$, for any $b\in \R$, we have
	\begin{equation*}
	\begin{split}
		\E\Big[\max\big\{(Y-b),0\big\}\Big] &= \int_b^{+\infty} (y-b)\phi(y)\text{d}y \\
		&= \int_b^{+\infty}(y-b)\text{d}\Big(-S(y)\Big)\\
		&= -(y-b)S(y)\Big\vert_b^{+\infty} + \int_b^{+\infty}S(y)\text{d}(y-b)\\
		&= \int_b^{+\infty}S(y)\text{d}y= \int_b^{+\infty}\Big(1-\Phi(y)\Big)\text{d}y= \int_b^{+\infty}\Phi(-y)\text{d}y\\
		&=\int_{-\infty}^{-b}\Phi(y)\text{d}y\\
		&=\tau(-b).
	\end{split}
	\end{equation*}
	
	Applying this into $X$, we have
	\begin{equation*}
	\begin{split}
		\E\Big[\max\big\{X-a,0\big\}\Big] &=\E\Big[\max\big\{\sigma Y+\mu-a,0\big\}\Big]\\
		&=\sigma\E\Big[\max\big\{Y-\frac{a-\mu}{\sigma},0\big\}\Big]\\
		&=\sigma\tau\Big(\frac{\mu-a}{\sigma}\Big).
	\end{split}
	\end{equation*}	

\end{proof}

\section{Proof of Theorem \ref{main}}
We will prove the main theorem by the following steps:
\begin{itemize}
	\item We will give a bound for \textbf{EI} in Lemma \ref{lemma2};
	\item Using the bound of \textbf{EI}, we will bound \textbf{Regret} by the sum of variance of posterior distribution. This will be presented in Lemma \ref{lem: variance}.
	\item Using Lemma \ref{lem: 4}, we will convert the sum of variance into \MIF, and then prove the main theorem.
\end{itemize}

Let $\hat{t}(x)$ denote the time when model $x$ began to be tested. Since every model will just be tested once, this is well-defined. 
Define $\hat{\sigma}(x) = \sigma_{\hat{t}(x)}(x)$. And let $x_{\text{test}}^{(1)}$, $x_{\text{test}}^{(2)}, \cdots, x_{\text{test}}^{(\hat{p}(T))}$ denote all models that have been tested or are being tested up to time $T$ by the order of their start-testing time.

It follows from Assumption 1 that we have $|z(x)-\mu_0(x)|\le \sigma_0(x)R$. So, the reward of all models has a universal bound $C_R$, which means $|z(x)|\le C_R$ for every model $x$.

\begin{lemma} \label{lemma2}
	Suppose Assumption \ref{assumption} holds. For each user $i$, let x be a model belonging to user $i$. Let $\Big(z(x)-z\big(x_i^*(t)\big)\Big)^+=\max\big\{z(x)-z\big(x_i^*(t)\big),0\big\}$. Then for every $t\ge 0$, we have
	\begin{equation*}
		\frac{\tau(-R)}{\tau(R)}\Big(z(x)-z\big(x_i^*(t)\big)\Big)^+\le \textbf{EI}_{i,t}(x)\le \Big(z(x)-z\big(x_i^*(t)\big)\Big)^++(R+1)\sigma_t(x).
	\end{equation*}
\end{lemma}

\begin{proof}[\textbf{Proof}]
	If $\sigma_t(x)=0$, then $|z(x)-\mu_t(x)|\le 0$, which means $z(x)=\mu_t(x)$
  is a constant for fixed $x$ and $t$. Then from (\ref{EI}), we have
  $\textbf{EI}_{i,t}(x) = \Big(z(x)-z\big(x_i^*(t)\big)\Big)^+$,  and the result is trivial.
	
	Suppose $\sigma_t(x)>0$. From (\ref{EI}) and Lemma \ref{lemma1}, we have $\textbf{EI}_{i,t}(x) = \sigma_t(x)\tau\Big(\frac{\mu_t(x)-z(x_i^*(t))}{\sigma_t(x)}\Big)$. Also, since $|z(x)-\mu_t(x)|\le R\sigma_t(x)$, we have
	\begin{equation} \label{compare}
		\Big|\frac{\mu_t(x)-z\big(x_i^*(t)\big)}{\sigma_t(x)}-\frac{z(x)-z\big(x_i^*(t)\big)}{\sigma_t(x)}\Big|\le R,
	\end{equation}

	which implies
	\begin{equation*}
		\frac{\mu_t(x)-z\big(x_i^*(t)\big)}{\sigma_t(x)} \le \frac{z(x)-z\big(x_i^*(t)\big)}{\sigma_t(x)}+R \le \frac{\Big(z(x)-z\big(x_i^*(t)\big)\Big)^+}{\sigma_t(x)}+R.
	\end{equation*}

	Also, since $\tau'(y) = \Phi(y)\in[0,1]$, $\tau$ is non-decreasing, and $\tau(y)\le 1+y$ for $y\ge 0$.  Therefore, 
	\begin{equation*}
	\begin{split}
		\textbf{EI}_{i,t}(x)&=\sigma_t(x)\tau\Big(\frac{\mu_t(x)-z\big(x_i^*(t)\big)}{\sigma_t(x)}\Big)\\
		&\le \sigma_t(x)\tau\Bigg(\frac{\Big(z(x)-z\big(x_i^*(t)\big)\Big)^+}{\sigma_t(x)}+R\Bigg)\\
		&\le \sigma_t(x)\Bigg(\frac{\Big(z(x)-z\big(x_i^*(t)\big)\Big)^+}{\sigma_t(x)}+R+1\Bigg)\\
		&=\Big(z(x)-z\big(x_i^*(t)\big)\Big)^++(R+1)\sigma_t(x),
	\end{split}
\end{equation*}
and the upper bound is proved for the case that $\sigma_t(x)>0$.
	
	If $\Big(z(x)-z\big(x_i^*(t)\big)\Big)^+=0$, then the left side of the inequality is 0. From (\ref{EI}), we have $\textbf{EI}_{i,t}(x)\ge 0$. Therefore, the lower bound is trivial.
	
	Now suppose $\Big(z(x)-z\big(x_i^*(t)\big)\Big)^+>0$, which means $\Big(z(x)-z\big(x_i^*(t)\big)\Big)^+=z(x)-z\big(x_i^*(t)\big)$.
	
	From (\ref{compare}), we have $\frac{\mu_t(x)-z\big(x_i^*(t)\big)}{\sigma_t(x)}\ge \frac{z(x)-z\big(x_i^*(t)\big)}{\sigma_t(x)}-R$. Also, we have $\tau(y)=y+\tau(-y)\ge y$. Thus
	\begin{equation*}
	\begin{split}
		\textbf{EI}_{i,t}(x)&=\sigma_t(x)\tau\Big(\frac{\mu_t(x)-z\big(x_i^*(t)\big)}{\sigma_t(x)}\Big)\\
		&\ge \sigma_t(x)\tau\Big(\frac{z(x)-z\big(x_i^*(t)\big)}{\sigma_t(x)}-R\Big)\\
		&\ge \sigma_t(x)\Big(\frac{z(x)-z\big(x_i^*(t)\big)}{\sigma_t(x)}-R\Big)\\
		&=z(x)-z\big(x_i^*(t)\big)-R\sigma_t(x).
	\end{split}
	\end{equation*}
	If $z(x)-z\big(x_i^*(t)\big)-R\sigma_t(x)\ge \frac{\tau(-R)}{\tau(R)}\Big(z(x)-z\big(x_i^*(t)\big)\Big)$, we conclude the proof. Otherwise, we have 
	\begin{equation*}
	\begin{split}
		R\sigma_t(x)&>\Big(1-\frac{\tau(-R)}{\tau(R)}\Big)\Big(z(x)-z\big(x_i^*(t)\big)\Big)\\
		&=\frac{\tau(R)-\tau(-R)}{\tau(R)}\Big(z(x)-z\big(x_i^*(t)\big)\Big)\\
		&=\frac{R}{\tau(R)}\Big(z(x)-z\big(x_i^*(t)\big)\Big),
	\end{split}
	\end{equation*}
	which implies $1>\frac{z(x)-z\big(x_i^*(t)\big)}{\sigma_t(x)\tau(R)}$. Also,
  it follows from \eqref{compare}  and the assumption
  $z(x)-z\big(x_i^*(t)\big)>0$ that
  \[
    \frac{\mu_t(x)-z\big(x_i^*(t)\big)}{\sigma_t(x)}+R\ge
    \frac{z(x)-z\big(x_i^*(t)\big)}{\sigma_t(x)}>0, \]
  and thus $\frac{\mu_t(x)-z\big(x_i^*(t)\big)}{\sigma_t(x)}\ge -R$. So
	\begin{equation*}
	\begin{split}
		\textbf{EI}_{i,t}(x)&=\sigma_t(x)\tau\Big(\frac{\mu_t(x)-z\big(x_i^*(t)\big)}{\sigma_t(x)}\Big)\\
		&\ge \sigma_t(x)\tau(-R)\ge \sigma_t(x)\tau(-R)\frac{z(x)-z\big(x_i^*(t)\big)}{\sigma_t(x)\tau(R)}\\
		&=\frac{\tau(-R)}{\tau(R)}\Big(z(x)-z\big(x_i^*(t)\big)\Big),
	\end{split}
	\end{equation*}
	which also concludes the proof.	

\end{proof}

\begin{lemma}\label{lem: variance}
	Under Assumption~\ref{assumption}, we have 
	\begin{equation*}
		\text{\bf Regret}_T\le \bigg(\frac{\tau(R)N(R+1)}{\tau(-R)M}\sum_{k=1}^{\hat{p}(T)}\hat{\sigma}(x_{\text{test}}^{(k)})+C_R+\frac{\tau(R)}{\tau(-R)}NC_R\bigg)\sum\limits_{i=1}^N c(x_i^*).
	\end{equation*}
\end{lemma}
\begin{proof}

Let $\hat{x}_j(t)$ denote the model that device $j$ is testing at time $t$. For each device $j$, define a function $f_j(t)$ as follows:
\begin{equation*}
	f_j(t) = \sum\limits_{i=1}^N \mathbbm{1}(\hat{x}_j(t)\in \mathcal{L}_i)\frac{\max\bigg\{z\big(\hat{x}_j(t)\big)-z\Big(x_i^*\big(\hat{t}(\hat{x}_j(t))\big)\Big), 0\bigg\}+(R+1)\hat{\sigma}(\hat{x}_j(t))}{c(\hat{x}_j(t))}.
\end{equation*}
For each user $i$, define a function $g_i(t)$ as follows:
\begin{displaymath}
	g_i(t) = \left\{ \begin{array} {ll}
		0, & \textrm{if the actual optimal model $x_i^*$ is being tested}\\
		z(x_i^*)-z\big(x_i^*(t)\big), & \textrm{otherwise}.
		\end{array} \right.
\end{displaymath}

First, we will prove that, for any $i=1,2,\cdots, N, j=1,2,\cdots, M$, and any $t\ge 0$, we have
\begin{equation}  \label{regrettrans}
	g_i(t)\le \frac{\tau(R)}{\tau(-R)}c(x_i^*)f_j(t).
\end{equation}
At time $t$, for any device $j$ and any user $i$, if model $x_i^*$ is being
tested, then $g_i(t)=0$. Thus the inequality (\ref{regrettrans}) is obvious.
If the test of model $x_i^*$ has been finished, then $x_i^*(t) = x_i^*$, we also
have $g_i(t)=0$, which implies (\ref{regrettrans}).

Therefore, we only consider the situation that $x_i^*$ has not been tested up to
time $t$ yet.  Then by Algorithm~\ref{alg1} we have
$\textbf{EIrate}_{\hat{t}(\hat{x}_j(t))}(x_i^*)\le
\textbf{EIrate}_{\hat{t}(\hat{x}_j(t))}(\hat{x}_j(t))$,  or equivalently
\begin{equation*}
	\frac{\textbf{EI}_{\hat{t}(\hat{x}_j(t))}(x_i^*)}{c(x_i^*)}\le \frac{\textbf{EI}_{\hat{t}(\hat{x}_j(t))}(\hat{x}_j(t))}{c(\hat{x}_j(t))}.
\end{equation*}
Applying Lemmas \ref{lemma2} and (\ref{EIsum}), we have
\begin{equation*}
\begin{split}
	\frac{\frac{\tau(-R)}{\tau(R)}\Big(z(x_i^*)-z\big(x_i^*(\hat{t}(\hat{x}_j(t)))\big)\Big)}{c(x_i^*)}&\le \frac{\textbf{EI}_{\hat{t}(\hat{x}_j(t))}(x_i^*)}{c(x_i^*)}\\
	&\le \frac{\textbf{EI}_{\hat{t}(\hat{x}_j(t))}(\hat{x}_j(t))}{c(\hat{x}_j(t))}\\
	&=\sum\limits_{i=1}^N \mathbbm{1}(\hat{x}_j(t)\in \mathcal{L}_i)\frac{\textbf{EI}_{i,\hat{t}(\hat{x}_j(t))}(\hat{x}_j(t))}{c(\hat{x}_j(t))}\\
	&\le f_j(t).
\end{split}
\end{equation*}
Noticing that $z(x_i^*(t))$ is a non-decreasing function of variable $t$ and $\hat{t}(\hat{x}_j(t))\le t$, we have $z\big(x_i^*(\hat{t}(\hat{x}_j(t)))\big)\le z(x_i^*(t))$, then we get (\ref{regrettrans}).
\newline

Recall that $c(x_i^*)$ is the time required to finish testing the model $x_i^*$,
so $g_i(t)\not=z(x_i^*)-z(x_i^*(t))$ only holds on a set with measure at most $c(x_i^*)$. From Assumption \ref{assumption}, $z(x_i^*)-z(x_i^*(t))\le C_R$. Therefore, for any $j=1,2,\cdots, M$, we have
\begin{equation*}
	\int_0^T  \Big(z(x_i^*)-z\big(x_i^*(t)\big)\Big)\text{d}t\le \int_0^T  g_i(t)\text{d}t + c(x_i^*)C_R\le \frac{\tau(R)}{\tau(-R)}c(x_i^*)\int_0^T f_j(t)\text{d}t + c(x_i^*)C_R,
\end{equation*}
which implies
\begin{equation} \label{c}
	\int_0^T  \Big(z(x_i^*)-z\big(x_i^*(t)\big)\Big)\text{d}t\le \frac{\tau(R)c(x_i^*)}{\tau(-R)M}\sum\limits_{j=1}^M \int_0^T f_j(t)\text{d}t+ c(x_i^*)C_R.
\end{equation}

Now, we fix $j=1$. In fact, $f_1(t)$ is a step function, that is, when device $1$ is testing a model, $f_1(t)$ remains the same. Let $x_i^{(1)}, x_i^{(2)},x_i^{(3)},\cdots$ denote the models tested on device $i$ by order. Let $x_i^{(p_i(T))}$ denote the model that is being tested on device $i$ at time $T$. Then we have
\begin{equation} \label{a}
\begin{split}
	\int_0^T f_1(t)\text{d}t &= \sum\limits_{k=1}^{p_1(T)-1} \sum\limits_{i=1}^N \mathbbm{1}(x_1^{(k)}\in \mathcal{L}_i)\Bigg(\max\bigg\{z\big(x_1^{(k)}\big)-z\Big(x_i^*\big(\hat{t}(x_1^{(k)})\big)\Big), 0\bigg\}+(R+1)\hat{\sigma}(x_1^{(k)})\Bigg)\\
	&\relphantom{=} {}+\int_{\hat{t}\big(x_1^{(p_1(T))}\big)}^T f_1(t)\text{d}t\\
	&\le \sum\limits_{k=1}^{p_1(T)} \sum\limits_{i=1}^N \mathbbm{1}(x_1^{(k)}\in \mathcal{L}_i)\Bigg(\max\bigg\{z\big(x_1^{(k)}\big)-z\Big(x_i^*\big(\hat{t}(x_1^{(k)})\big)\Big), 0\bigg\}+(R+1)\hat{\sigma}(x_1^{(k)})\Bigg)\\
	&=\sum\limits_{i=1}^N \sum\limits_{k=1}^{p_1(T)} \mathbbm{1}(x_1^{(k)}\in \mathcal{L}_i)\max\bigg\{z\big(x_1^{(k)}\big)-z\Big(x_i^*\big(\hat{t}(x_1^{(k)})\big)\Big), 0\bigg\}+N(R+1)\sum\limits_{k=1}^{p_1(T)} \hat{\sigma}(x_1^{(k)}).
\end{split}
\end{equation}

Let $x_{1,i}^{(1)}, x_{1,i}^{(2)}, \cdots, x_{1,i}^{(p_{1,i}(T))}$ denote the models from user $i$ tested on device $1$ by order. Then, we have
\begin{equation*}
\begin{split}
	&\relphantom{=} {}\sum\limits_{k=1}^{p_1(T)} \mathbbm{1}(x_1^{(k)}\in \mathcal{L}_i)\max\bigg\{z\big(x_1^{(k)}\big)-z\Big(x_i^*\big(\hat{t}(x_1^{(k)})\big)\Big), 0\bigg\}\\
	&= \sum_{k=1}^{p_{1,i}(T)} \max\bigg\{z\big(x_{1,i}^{(k)}\big)-z\Big(x_i^*\big(\hat{t}(x_{1,i}^{(k)})\big)\Big), 0\bigg\}\\
	&=\sum_{k=1}^{p_{1,i}(T)} \bigg(\max\bigg\{z\big(x_{1,i}^{(k)}\big), z\Big(x_i^*\big(\hat{t}(x_{1,i}^{(k)})\big)\Big)\bigg\} - z\Big(x_i^*\big(\hat{t}(x_{1,i}^{(k)})\big)\Big)\bigg).
\end{split}
\end{equation*}

When $x_{1,i}^{(k+1)}$ begins to test, $x_{1,i}^{(k)}$ of course has finished testing, so
\begin{equation*}
	z\Big(x_i^*\big(\hat{t}(x_{1,i}^{(k+1)})\big)\Big)\ge \max\bigg\{z\big(x_{1,i}^{(k)}\big), z\Big(x_i^*\big(\hat{t}(x_{1,i}^{(k)})\big)\Big)\bigg\}\ge z\Big(x_i^*\big(\hat{t}(x_{1,i}^{(k)})\big)\Big).
\end{equation*}
Therefore, 
\begin{equation} \label{b}
\begin{split}
	&\relphantom{=} {}\sum\limits_{k=1}^{p_1(T)} \mathbbm{1}(x_1^{(k)}\in \mathcal{L}_i)\max\bigg\{z\big(x_1^{(k)}\big)-z\Big(x_i^*\big(\hat{t}(x_1^{(k)})\big)\Big), 0\bigg\}\\
	&\le \max\bigg\{z\big(x_{1,i}^{(p_{1,i}(T))}\big), z\Big(x_i^*\big(\hat{t}(x_{1,i}^{(p_{1,i}(T))})\big)\Big)\bigg\} - z\Big(x_i^*\big(\hat{t}(x_{1,i}^{(1)})\big)\Big)\\
	&\le C_R.
\end{split}
\end{equation}

Applying (\ref{b}) into (\ref{a}), we get
\begin{equation*}
	\int_0^T f_1(t)\text{d}t\le NC_R+N(R+1)\sum\limits_{k=1}^{p_1(T)} \hat{\sigma}(x_1^{(k)}).
\end{equation*}
Similarly, for any $j \in \{ 1,2,\cdots, M\}$, we also have
\begin{equation*}
	\int_0^T f_i(t)\text{d}t\le NC_R+N(R+1)\sum\limits_{k=1}^{p_i(T)} \hat{\sigma}(x_i^{(k)}).
\end{equation*}
Applying these to (\ref{c}), we have
\begin{equation*}
	\int_0^T  \Big(z(x_i^*)-z\big(x_i^*(t)\big)\Big)\text{d}t\le \frac{\tau(R)}{\tau(-R)}c(x_i^*)NC_R + c(x_i^*)C_R + \frac{\tau(R)c(x_i^*)N(R+1)}{\tau(-R)M}\sum_{k=1}^{|\mathcal{L}(t)|}\hat{\sigma}(x_{\text{test}}^{(k)}),
\end{equation*}
where $|\mathcal{L}(t)| = \sum\limits_{j=1}^{M} p_j(T)$.

Applying this into (\ref{regret}), we get the result:
\begin{equation*}
	\textbf{Regret}_T\le \bigg(\frac{\tau(R)}{\tau(-R)}NC_R+C_R+\frac{\tau(R)N(R+1)}{\tau(-R)M}\sum_{k=1}^{|\mathcal{L}(t)|}\hat{\sigma}(x_{\text{test}}^{(k)})\bigg)\sum\limits_{i=1}^N c(x_i^*).
\end{equation*}

\end{proof}

\begin{lemma} \label{lem: 4}
	Suppose $A$ is an $n\times n$ positive definite matrix, and we divide it into 4 blocks, 
	$A=\left( \begin{array}{cc}
	         	A_{n-1} & B \\
			B^\mathrm{T} & a
			\end{array} \right)$, here $A_{n-1}$ is an $(n-1)\times (n-1)$ matrix. Then we have
	\begin{equation*}
		\frac{\det(A)}{\det(A_{n-1})} = a-B^\mathrm{T}A_{n-1}^{-1}B.
	\end{equation*}

\end{lemma}
\begin{proof}
	We have the following
	\begin{equation*}
		A\left(\begin{array}{cc}
			I_{n-1} & -A^{-1}B \\
			0 & 1
			\end{array} \right) = \left( \begin{array}{cc}
	         	A_{n-1} & B \\
			B^\mathrm{T} & a
			\end{array} \right) \left(\begin{array}{cc}
			I_{n-1} & -A^{-1}B \\
			0 & 1
			\end{array} \right) = \left( \begin{array}{cc}
			A_{n-1} & 0 \\
			B^\mathrm{T} & a-B^\mathrm{T}A^{-1}B
			\end{array} \right),
	\end{equation*}
	
	here, $I_{n-1}$ is an $(n-1)\times (n-1)$ identity matrix. Computing the determinant of both sides of the equality, we get:
	\begin{equation*}
		\det(A) = \det(A_{n-1})(a-B^\mathrm{T}A_{n-1}^{-1}B).
	\end{equation*}
\end{proof}

\begin{proof} [\textbf {Proof of Theorem \ref{main}}]
	We only have to prove that 
	\begin{equation*}
		\sum_{s=1}^
        {|\mathcal{L}(t)|}\hat{\sigma}(x_{\text{test}}^{(s)})\le M + \text{\bf MIU}(T,K).
	\end{equation*}
	Then, combing with Lemma \ref{lem: variance}, we complete the proof.
	
	For model $x_{\text{test}}^{(s)}$ $(s>M)$, when this model begins to test, there are $s-1$ models having finished testing or is testing. There are $M$ devices, and $x_{\text{test}}^{(s)}$ should occupy one device, so at least $s-1-(M-1) = s-M\ge 1$ models having finished testing. Suppose these finish-testing models are $x_{f_1},x_{f_2}, \cdots, x_{f_{s-M}}$.
	
	Let $P$ denotes the variance matrix of random variable $z(x_{f_1}), z(x_{f_2}), \cdots, z(x_{f_{s-M}}), z\big(x_{\text{test}}^{(s)}\big)$ (rows and columns are arranged by this order). Notice that $P$ is an $(s+1-M)\times (s+1-M)$ matrix.
	
	From (\ref{v}), we know that the conditional variance is smaller than unconditional variance for a multivariable Gaussian distribution (or Gaussian process), so we have
	\begin{equation}\label{eq: s}
		\hat{\sigma}(x_{\text{test}}^{(s)}) = \sigma\big(x_{\text{test}}^{(s)}\mid x_{\text{test}}^{(1)}, x_{\text{test}}^{(2)}, \cdots, x_{\text{test}}^{(\hat{p}(T)-1)}\big)\le \sigma\big(x_{\text{test}}^{(s)}\mid x_{f_1},x_{f_2}, \cdots, x_{f_{s-M}}\big).
	\end{equation}
	
	Let's rewrite $P$ as $\left(\begin{array}{cc}
			P_1 & B \\
			B^\mathrm{T} & p
			\end{array} \right)$, where $P_1$ is an $(s-M)\times (s-M)$ matrix. Also from (\ref{v}), we have $\sigma\big(x_{\text{test}}^{(k)}\mid x_{f_1},x_{f_2}, \cdots, x_{f_{s-M}}\big) = p-B^\mathrm{T}P_1^{-1}B$. From Lemma \ref{lem: 4}, we have: $p-B^\mathrm{T}P_1^{-1}B = \frac{\det(P_1)}{\det(P)}$. From the definition of $\textbf{MIU}_{s-M+1}(K)$, we have $\frac{\det(P_1)}{\det(P)}\le\MIF_{s-M+1}(K)$. Together with (\ref{eq: s}), we obtain
	\begin{equation}\label{not special}
		\hat{\sigma}(x_{\text{test}}^{(s)})\le\MIF_{s-M+1}(K).
	\end{equation}
	Particularly, because $s>M$, we have $s-M+1\ge 2$.
	
	For model $x_{\text{test}}^{(s)}$ $(s\le M)$. Again, because the conditional variance is smaller than unconditional variance for a multivariable Gaussian distribution(or Gaussian process), we have
	\begin{equation}\label{special}
		\hat{\sigma}(x_{\text{test}}^{(s)})\le \sigma(x_{\text{test}}^{(s)})\le 1.
	\end{equation}
	
	From (\ref{not special}) and (\ref{special}), we conclude that
	\begin{equation*}
		\sum_{s=1}^{\hat{p}(T)}\hat{\sigma}(x_{\text{test}}^{(s)})\le M + \text{\bf MIU}(T,K),
	\end{equation*}
	which complete the proof as analyzing at the beginning of this proof.
\end{proof}


\end{document}